\documentclass[journal]{IEEEtran}

\IEEEoverridecommandlockouts                              

\let\labelindent\relax
\usepackage{amsmath,graphicx,multicol}
\usepackage{amsfonts}
\usepackage{color,xcolor}
\usepackage{subcaption}
\usepackage{bbm}
\usepackage{cite}
\usepackage{amssymb}
\usepackage{tabularx,booktabs}
\usepackage{algorithm}
\usepackage{algpseudocode}
\usepackage[utf8]{inputenc}
\usepackage[T1]{fontenc}
\usepackage[english]{babel}
\usepackage{enumitem}

\usepackage{amsthm}
\usepackage{bm}
\usepackage{url,hyperref}
\usepackage{cleveref}
\usepackage{mathtools}
\newtheorem{definition}{Definition}[]
\newtheorem{assumption}[]{Assumption}

\newtheorem{proposition}{Proposition}[]
\newtheorem{theorem}{Theorem}[]

\newtheorem{lemma}[]{Lemma}

\DeclarePairedDelimiterX{\infdivx}[2]{(}{)}{%
  #1\;\delimsize\|\;#2%
}

\DeclareMathOperator{\E}{\mathbb{E}}

\def\E{\mathbb{E}}

\newcommand{\RNum}[1]{\uppercase\expandafter{\romannumeral #1\relax}}
\makeatletter
\renewcommand{\fnum@figure}{Fig.~\thefigure}
\makeatother

\title{
Submodular Maximization Approaches for Equitable Client Selection in Federated Learning 
}

\author{Andrés C. Castillo J., Ege C. Kaya, Lintao Ye, and Abolfazl Hashemi
\thanks{Andrés C. Castillo J., Ege C. Kaya, and Abolfazl Hashemi are with the Elmore Family School of Electrical and Computer Engineering at Purdue University, West Lafayette, IN, USA.  Email: \{castil12, kayae, abolfazl\}@purdue.edu. 
Lintao Ye is with the School of Artificial Intelligence
and Automation, Huazhong University of Science and Technology, Wuhan, China. Email: yelintao93@hust.edu.cn. 
Parts of the results in the paper were presented at the Conference on Decision and Control (CDC), Milan, Italy, December 2024 \cite{cdc-version}. 
}
}

\begin{document}

\maketitle
\thispagestyle{empty}
\pagestyle{empty}

\begin{abstract}


In a conventional Federated Learning framework, client selection for training typically involves the random sampling of a subset of clients in each iteration. However, this random selection often leads to disparate performance among clients, raising concerns regarding fairness, particularly in applications where equitable outcomes are crucial, such as in medical or financial machine learning tasks. This disparity typically becomes more pronounced with the advent of performance-centric client sampling techniques.
This paper introduces two novel methods, namely \textsc{SubTrunc} and \textsc{UnionFL}, designed to address the limitations of random client selection. Both approaches utilize submodular function maximization to achieve more balanced models. By modifying the facility location problem, they aim to mitigate the fairness concerns associated with random selection.
\textsc{SubTrunc} leverages client loss information to diversify solutions, while \textsc{UnionFL} relies on historical client selection data to ensure a more equitable performance of the final model. Moreover, these algorithms are accompanied by robust theoretical guarantees regarding convergence under reasonable assumptions.
The efficacy of these methods is demonstrated through extensive evaluations across heterogeneous scenarios, revealing significant improvements in fairness as measured by a client dissimilarity metric.

\end{abstract}

\section{Introduction}

\IEEEPARstart{T}{he} Federated Learning (FL) paradigm involves the collaborative training of a centralized machine learning model using edge devices, commonly referred to as clients. This setting allows for models to be trained using localized data from these devices without the need to transmit the data to a centralized location. Updates to the model are accumulated from these clients via periodic communication rounds and aggregated at the central location resulting in an improved machine learning model.

Traditionally, randomly selecting a subset of clients has been the de facto approach for this setting \cite{McMahan2016Communication-EfficientData}.
However, previous work has found that, oftentimes, this random selection approach does not perform well in terms of convergence and fairness properties, especially in heterogeneous settings where the data being held by each client may not necessarily come from the same distribution \cite{cho2020clientselectionfederatedlearning}. This is especially evident in applications characterized by a high degree of data heterogeneity, where the need for a balanced and fair machine learning model is highly prioritized, such as in the use of computer vision models in medical imaging, where, for example, the data may be MRI images that have been produced by machines from different manufacturers. Because of this, client selection remains an open challenge within the field \cite{Kairouz2019AdvancesLearning,Wang2021AOptimization}. Another motivating scenario is the case of networked sensing \cite{grime1994data,olfati2007distributed}. In networks comprising units, there is a common goal to devise an inference method that reduces the total estimation error. Nevertheless, in numerous scenarios, each unit must produce a dependable estimate to avoid impeding decision-making among other units in the network \cite{hashemi2018near}. For instance, in autonomous vehicle settings, a unit with substantial estimation error might necessitate slowing down, thus influencing the behavior of other units \cite{lu2014connected}. Hence, there is an urgent need to minimize the collective mean-square estimation error across the entire network while ensuring consistent performance among individual units. Such considerations have further received attention in shared communication systems \cite{jurdi2020scheduling}.

To address this, differing from previous client selection strategies, the idea of incorporating submodular set functions as a viable strategy for solving the client selection problem has been proposed\cite{ BalakrishnanDiverseMaximization, Ye2021ClientApproach, omori2023combinatorial, zhang2024addressing}. A typical problem studied in submodular optimization literature is the maximization of a submodular function under a cardinality constraint \cite{Asadpour2008StochasticMaximization,Krause2014SubmodularMaximization,NivBuchbinder2014SubmodularConstraints}. In this problem, the task is to maximize the utility of the selection made from a ground set $N$, while making sure that the number of elements in the selection stays under a given integer cardinality constraint value $\kappa$. This can be formalized as
\begin{equation}
\begin{gathered}
\max_{S \subseteq N}f(S) \label{cardinality}\qquad
\text{s.t.}\;\;|S| \leq \kappa,
\end{gathered}
\end{equation}
\cite{NivBuchbinder2014SubmodularConstraints} where $f$ is a submodular function, $N$ is the ground set, and $\kappa$ is a positive integer.


\subsection{Contributions}




This paper presents two novel algorithms, \textsc{SubTrunc} and \textsc{UnionFL}, specifically designed to enhance the fairness of client selection criteria within the context of Federated Learning. These algorithms employ a tailored submodular function approach derived from the facility location problem to optimize client selection.

Building upon the framework of \textsc{DivFL} \cite{BalakrishnanDiverseMaximization}, \textsc{SubTrunc} modifies the gradient similarity approach by integrating a truncated submodular function term. This addition serves as a regularization term aimed at encouraging balance by incorporating both the gradient similarity metrics and the individual loss values of clients. Consequently, the resulting model achieves a more equitable performance distribution across all clients.

Similarly, \textsc{UnionFL} introduces a regularization term that promotes client diversification. By maintaining a record of selected clients within a window of preceding time steps in each round, \textsc{UnionFL} promotes diverse solutions by encouraging the involvement of clients that have not been previously selected in a similar manner. Notably, this regularization term seamlessly integrates into conventional submodular maximization frameworks, offering a streamlined approach to promoting fairness in client selection.


Both of these proposed methods not only ensure the selection of the most representative clients but also facilitate their participation in a fair manner, leading to a model that exhibits uniform performance across all clients, both in terms of training and test outcomes.

Additionally, we show that these algorithms enjoy strong theoretical guarantees on their convergence under mild assumptions such as nonconvexity. In particular, under smoothness, and without assuming the well-known Bounded Client Dissimilarity (BCD) assumption where one assumes that for some $G^2 \in \mathbb{R}$,
\begin{equation}\label{BCD}
\lVert\nabla f_i(w) - \nabla f(w) \rVert^2 \leq G^2\; \forall\, w \in \mathbb{R}^d, i \in  N.
\end{equation}
Our work shows that by relaxing: \eqref{BCD}, which, in practice, is a hard-to-verify assumption \cite{RudrajitDas2022FasterMomentum}; the assumptions of a uniformly bounded gradient; strong convexity required by prior works; our proposed algorithm then enjoys strong theoretical guarantees that hold on the expected performance in nonconvex optimization problems. That is, our method needs $K = \mathcal{O}(1/\epsilon^2)$ rounds of communication in order to achieve $\E{\left[\|\nabla f\left(w_{k^*}\right)\|^2\right]} \leq \epsilon$ 
while doing away with the Bounded Client Dissimilarity Assumption under a smooth nonconvex scenario.

\subsection{Related Works}


Recently, by being able to model client selection through a submodular maximization problem, \cite{BalakrishnanDiverseMaximization} was able to obtain improved performance when compared with traditional client selection strategies such as random sampling. This work does so by attempting to select a subset of clients whose gradient most closely resembles that of the full client set, by modeling the problem as a submodular maximization problem, which can be solved with greedy methods. Although our work, similar to \cite{BalakrishnanDiverseMaximization}, utilizes the concept of submodular maximization for client selection, it differs from it in that we employ a facility location submodular function modified by a truncated submodular function, making use of the loss value of each client for the truncation. This modification acts as a fairness-aware term that promotes a balanced model performance across all clients regardless of the distribution of the data these clients may hold, resulting in models whose performance does not drastically differ from client to client throughout the training process. Additionally, our theoretical analysis of the convergence of our method is significantly different from \cite{BalakrishnanDiverseMaximization}  and utilizes considerably milder assumptions.
 
Both \cite{Ye2021ClientApproach, omori2023combinatorial} have also explored the use of submodular maximization and its greedy solution as a means to solve the client selection problem. In particular, these works seek to create an optimal client schedule under computational and time constraints. This approach differs from ours in that instead of approximating the full client gradient by a subset and greedily selecting them, the problem is modeled as a Submodular-Cost Submodular-Knapsack problem where the selection is constrained by computational and timing metrics, whereas we look at our constraint through a truncated approach, which can be likened to the notion of the presence of a budget. Additionally, neither\cite{Ye2021ClientApproach, omori2023combinatorial} establish any convergence rate for the resultant FL method.

In tackling data heterogeneity and client selection schemes within FL, \cite{zhang2024addressing} also approaches the client selection problem as a submodular maximization problem that can be greedily solved under a knapsack constraint. This work seeks to maximize statistical performance under system performance constraints, like upload and communication time. This differs from our work in that we leverage the heterogeneity of the data in each client to create more diverse solutions by exploiting the loss at each client's dataset while approximating the solution set via the client's gradient. Additionally, we do not further constrain the problem under system heterogeneity metrics as \cite{zhang2024addressing}.


 

\section{Preliminaries and Background}
This paper considers the standard FL setting comprised of a central server that acts as a model aggregator, $\lvert N \rvert$ clients that can participate in training, and a model parameterized by $w \in \mathbb{R}^d$. Each client $i$ in this setting has data coming from a distribution $D_i$ and an objective function $f_i(w)$ which is the expected loss of the client concerning some loss function $l$ over drawn data from $D_i$. The main objective is for the central server to optimize the average loss $f(w)$ over the $\lvert N \rvert$ clients:
\begin{equation}\label{FedAvg}
f(w):= \frac{1}{\lvert N \rvert} \sum_{i\in N} f_i(w),
\end{equation}
\begin{equation}\label{DataDist}
f_i(w) := \mathbb{E}_{x\sim D_i}[l(x,w)].
\end{equation}
When the data distributions across all clients are equal, the setting is considered independent and identically distributed (iid). If they differ across clients, then it is considered heterogeneous (non-iid).
 
At any given round of a typical FL setting, a random subset of clients is chosen to perform training. By carrying out a series of local gradient descent updates in the client's data and communicating these to the central server, the final model is constructed. This method of averaging out the client's updates is known as the \textsc{FedAvg} algorithm \cite{McMahan2016Communication-EfficientData}. 
 
However, client selection still remains an open challenge within FL \cite{Kairouz2019AdvancesLearning,Wang2021AOptimization}. 
Recently, utilizing submodular functions for client selection by modeling the problem as a facility location problem was introduced\cite{BalakrishnanDiverseMaximization}. This strategy aims to find a representative subset of clients whose aggregated update models what the overall aggregated update would look like if all clients participated in the training.
 
Our proposed methods build up on the idea of leveraging submodular functions to create more representative client sets. It does so by finding a representative subset of clients while ensuring fair client usage by allowing those clients who may not be the most representative in a given round to have more opportunity to participate in training. This fair client usage is introduced through a new regularization term that leverages the client's most recent loss value to design a judicious truncated function and adds that to the original modeling of the client selection problem via submodular functions.
 
Let us now introduce the concepts of marginal gain and submodularity \cite{Krause2014SubmodularMaximization}.

\begin{definition}[Marginal gain]\label{MGDef}
Given a set function $f\vcentcolon2^N\rightarrow\mathbb{R}$ and $A, B \subseteq N$, we denote $f(B \cup A) - f(B)$, the \textit{marginal gain} in $f$ due to adding $A$ to $B$, by $\Delta(A|B)$. When the set $A$ is a singleton, i.e., $A = \{a\}$, we drop the curly brackets to adopt the short-hand notation $\Delta(a|B)$.  
\end{definition} 

We usually think of $f$ as assigning a \textit{utility score} to each subset $A \subseteq N$.
 
With Definition \ref{MGDef}, we can now introduce the concept of submodularity in set functions:

\begin{definition}[Submodularity]\label{Submodularity}
A set function $f\vcentcolon2^N\rightarrow\mathbb{R}$ is \textit{submodular} if for every $A\subseteq B \subseteq N$ and $e \in N \setminus B,$ it holds that
\begin{equation}\label{MGEq}
\Delta(e|A) \geq \Delta(e|B).
\end{equation}
\end{definition}

This definition of submodularity gives a clear intuition about the nature of submodular functions, showcasing the diminishing marginal gains property, which can be exploited in the context of client selection \cite{Krause2014SubmodularMaximization}.
 
Additionally, let us introduce monotone submodular functions.

\begin{definition}[Monotonicity]\label{Monotonicity}
    Let $f(S)$ be a submodular set function. If $f(S)$ satisfies the following: $\forall \, A \subseteq B \subseteq N, f(A) \leq f(B)$. It is said that $f$ is a monotone submodular set function.
\end{definition}

It is well known that if $f(S)$ is a submodular function, $g(f(S))$ is also submodular for any concave function $g$ \cite{Krause2014SubmodularMaximization}. This result leads to the following proposition which establishes the monotonicity and submodularity of \textit{truncated} functions \cite{Balcan2018SubmodularOptimization}.

\begin{proposition}[Truncation]\label{TruncatedFunctions}
Let $g(S)$ be a monotone submodular set function composed of a non-decreasing concave function and let $b \in \mathbb{R}^+$. Then, $g$ remains monotone submodular under truncation, i.e.,
\begin{equation}
    f(S) := \min\{g(S), b\},
\end{equation}
is a monotone submodular function. 
\end{proposition}

We further have the following proposition \cite{Krause2014SubmodularMaximization}.
\begin{proposition}[Linear combination]\label{LinearCombo}
    Nonnegative linear combinations of submodular functions preserve submodularity. More formally, let $g_1, g_2, ..., g_n : 2^N \rightarrow \mathbb{R}$ be submodular set functions. Let $\alpha_1, \alpha_2, ..., \alpha_n \geq 0$. Then, the positive linear combination
    \begin{equation}\label{LinearComboEq}
    f(S) := \sum_{i=1}^n \alpha_i g_i(S),
    \end{equation}
    is also a submodular function.
\end{proposition}

Submodular functions and their maximization are very useful in practice, with a wide range of applications 
\cite{AndreasKrause2008RobustSelection, Shamaiah2010GreedySubmodularity, AndreasKrause2007Near-optimalFunctions, LingyaLiu2023GreedyEllipsoid, Mossel2007OnNetworks,Dughmi2009RevenueSubmodularity}. 
Due to their versatility and natural occurrence in practical and well-known settings 
their optimization has garnered interest in fields like optimization, control systems, signal processing and machine learning \cite{MarkusPuschel2020DiscreteFunctions,MarioCoutino2018SubmodularObservations,Ye2020DistributedFunctions,Hespanha2020OptimalOptimization,Edmonds2003SubmodularPolyhedra}. 

 


Modeling client selection as a modified facility location problem with a truncated submodular function regularization serves as the basis for our proposed method for a more balanced client selection approach.

Lastly, we talk about a concept that is closely related to submodularity, and can be seen as its mirror image, much like the relation between concavity and convexity:
\begin{definition}[Supermodularity]\label{Supermodularity}
A set function $f\vcentcolon2^N\rightarrow\mathbb{R}$ is \textit{supermodular} if for every $A\subseteq B \subseteq N$ and $e \in N \setminus B,$ it holds that
\begin{equation}\label{MGEq2}
\Delta(e|A) \leq \Delta(e|B).
\end{equation}
\end{definition}
It can be seen that the only difference between the definitions of submodularity and supermodularity is the direction of the inequality. It is simple to conclude, then, that the additive inverse of any submodular function is supermodular, and vice versa.

\section{SubTrunc and UnionFL: Fairness-aware Client Selection Approaches}
In this section, we develop our proposed fairness-aware client selection approaches.
\subsection{SubTrunc}
To motivate our formulation of a modified facility location objective via a truncated submodular function, we first follow the outline of DivFL \cite{BalakrishnanDiverseMaximization}, which follows the logic found in \cite{Mirzasoleiman2019CoresetsModels}. Suppose there is a mapping $\sigma: N \rightarrow S$, in which $N$ is the ground set of elements and $S$ is the constructed set and where the gradient information $\nabla f_{i}(w$) from client $i$ is approximated by the gradient information from a selected client $\sigma(j)\in S$. For $j \in S$, $C_j := \{i \in N | \sigma(i)=j\}$ is the set approximated by client $j$ and $\gamma_j:=|C_j|$ results in the following formulation: 
\begin{equation}\label{GradientApprox}
\begin{gathered}
\min_{S\subseteq N}\lVert\sum_{i\in N} \nabla f_i(w) - \sum_{j \in S}\gamma_j \nabla f_j(w)\rVert \leq \\
\sum_{i\in N} \min_{j\in S} \lVert\nabla f_i(w) - \nabla f_j(w)\rVert := \bar{G}(S).
\end{gathered}
\end{equation}
That is, $\sum_{j \in S}\gamma_j \nabla f_j(w)$ can be viewed as the approximation of the global gradient $\sum_{i\in N} \nabla f_i(w)$. Therefore, the left-hand side of \eqref{GradientApprox} can be interpreted as the approximation error, and the right-hand side of \eqref{GradientApprox} provides an upper bound on the approximation error. Thus, to minimize the approximation error, \textsc{DivFL} aims to select a set of clients $S$ that minimizes $\bar{G}(S)$ subject to a cardinality constraint on S. Upon defining ${G}(S):= \bar{G}(\emptyset) - \bar{G}(S)$ this task can be written as
\begin{equation}\label{eq:divfl-selection}
    \max_{S\subseteq N} G(S) \quad \text{s.t.} \quad |S|\leq \kappa,
\end{equation}
where $\kappa$ is a target bound on the number of participating clients in each communication round.
Minimizing $\bar{G}(S)$ or equivalently maximizing $G(S)$ is the equivalent of maximizing the well-known facility location monotone submodular function \cite{Krause2014SubmodularMaximization}. However, this problem, under a cardinality constraint, is NP-hard in general which requires efficient approximation algorithms to provide a near-optimal solution. The greedy algorithm and its randomized variants are among the canonical methods for such optimization problems \cite{Nemhauser1978AnFunctionsI, Mirzasoleiman2014LazierGreedy, Hashemi2021RandomizedSubmodularity,Hashemi2022OnMaximization}.

Finding the most representative set at any given round may not always provide a model that performs in a balanced and similar fashion across all clients thereby leading to potential unfair behaviours in FL-based model training. As a result of this observation, 
we propose a fairness regularization term utilizing the truncation of a judicious submodular function:
\begin{equation}\label{FairnessTerm}
H(S) := \lambda\min(b, F(S)),
\end{equation}
with 
\begin{equation}\label{LnSub}
F(S):= \sum_{i\in S}\phi(f_i(w)),
\end{equation}
where $f_i(w)$ is the expected loss of client $i$ with respect to some loss function $l$ as defined in \eqref{DataDist}, $\lambda \ge 0$ and $b \in \mathbb{R}^+$ are input parameters aiming to explore the inherent trade-off between performance and fairness; $\phi$ can be any monotone nondecreasing function with a bounded Lipschitz constant $L$. Tying client loss to this regularization term enhances a diversified client selection process. Changes in $b$ allow for further client participation, especially for those clients with minimal participation, thus ensuring a more fair total client participation, where each client can have an opportunity to contribute to the final model. When a function with $L <1$ is used, $F(S)$ effectively becomes an attenuation term, suppressing the difference between the client losses. Whereas, when $L > 1$, $F(S)$ enhances this difference. 
However, the choice of $\phi$ could be a potential avenue of further research as finding a $\phi$ that judiciously tunes this attenuation or enhancement effect could be of interest. Here, we assume without loss of generality that the clients' loss functions are nonnegative. The combination of this fairness-aware term with the original objective results in the following optimization for client selection:
\begin{equation}\label{FinalComposition}
\max_{S\subseteq N} W(S):= G(S) + H(S) \quad \text{s.t.} \quad  |S|\leq \kappa.
\end{equation}


A desirable property of the proposed formulation is the preservation of monotonicity and submodularity, which is indeed the case, as demonstrated next.
\begin{proposition}
\label{SubmodularityProp}
    The set function $W(S)$ in \eqref{FinalComposition} is monotone and submodular.
\end{proposition}
\begin{proof}
Note that $F(S)$ is a modular function and hence monotone and submodular \cite{Krause2014SubmodularMaximization}. By Proposition \ref{TruncatedFunctions}, $H(S)$ is therefore a monotone submodular function. Finally,
from Proposition \ref{LinearCombo}, it can be seen that any linear combination of submodular functions with the same ground set remains submodular. Since both $G(S)$ and $H(S)$ in \eqref{FairnessTerm} are monotone submodular, their nonnegative linear combination in \eqref{FinalComposition} is also monotone submodular.
\end{proof}

The fairness-aware term as defined in \eqref{FairnessTerm} is the combination of tuneable parameters in the form of both $\lambda$ and $b$, where $\lambda$ represents a weighting on the fairness term and $b$ acts as a truncation parameter. From this formulation, we can see that if we set $\lambda = 0$, then we obtain the original \textsc{DivFL} formulation \cite{BalakrishnanDiverseMaximization}. On the other hand, increasing $\lambda$ puts more weight on the proposed regularization term, thereby prioritizing fairness over performance. Additionally, varying $b$, as stated previously, will allow \textit{training participation} to those clients who
have not participated in previous rounds, or whose participation has been minimal when compared to other clients. In particular, consider a scenario where $b$ is very large. Then, the minimum in the definition of $H(S)$ will typically equate to $F(S):= \sum_{i\in S}\phi(f_i(w))$ which is maximized by selecting the worst-performing clients, according to their local loss functions, who are suffering the most as a result of learning a collaborative model. On the other hand, if $b$ is very small, then the minimum in the definition of $H(S)$ will typically equate to $b$ which is independent of $S$, and this effectively makes the client selection independent of local performance and can be thought of as a global performance-centric formulation. 

\subsection{UnionFL}
We further propose the following novel formulation, which incorporates a regularization term involving the history of previously constructed solutions to the objective function, to promote diversification of the sequence of subsequent solutions. Additionally, this method can also be employed with any objective function so long as this objective function is submodular. Suppose $f_t \vcentcolon N_t \rightarrow \mathbb{R}^+$ is a submodular objective function that we wish to maximize at time step $t$, making a selection out of ground set $N_t$, selecting no more elements than $K_t$. We propose the regularization of this objective as follows:
\begin{equation}
\label{UnionEq}
\begin{gathered}
\max_{S_t \in N_t} f_t(S_t) - \mu g_t(S_t)\quad
\text{s.t. } \lvert S_t\rvert \leq K_t,
\end{gathered}
\end{equation}
where $g_t(S_t) \vcentcolon= \lvert (\bigcup_{i \in u_t} S_i )\cap S_t\rvert;$ $f_t\left(S_t\right),$ is the submodular function as defined by \eqref{eq:divfl-selection}, however it is worth pointing out that any submodular function fits into this framework; $u_t$ is a `look-back' window hyperparameter for previously constructed sets; $\mu \geq 0$ is a regularization hyperparameter and $t$, represents the concept of global rounds within the FL framework.

In this formulation the original objective function, $f_t(S_t)$, is penalized by, $\mu g_t\left(S_t\right)$, based on how many elements the solution set that is currently being constructed has in common with the previous solution sets constructed through time steps within a window $u_t$. For instance, at time step $t$, a typical value for $u_t$ might be $u_t = \{t-5, \ldots, t-1\}$ in which case, the penalization at each time step $t$ would depend on the common elements of not only the current solution but also  the common elements with the latest $5$ solutions. 

We can easily see that as $u_t$ increases, the original objective suffers a heavier penalty given that we will be looking into more previously constructed solution sets, which means that our common client pool will be larger. On the contrary, as $u_t$ decreases, the penalty will be less severe given that we will be looking into fewer previously constructed solution sets, resulting in a smaller common client pool.

Additionally, tuning $\mu$ allows for a finer control of the penalty incurred to the original objective function. That is, if we set $\mu = 0$, then we obtain a typical submodular maximization problem under a cardinality constraint. However, increasing $\mu$ ensures that the weight on the penalization is higher, thereby allowing increased client training participation to those clients who may not have participated previously.

It is intuitively clear to see that this novel formulation would be beneficial for the diversification of solutions, by incentivizing the selection process to pick elements that were unpicked within the time window $u_t$. However, what remains to be answered is how to produce a solution for it. Fortunately, the following result provides an answer.
\begin{proposition}\label{SuperSubmodular}
$g_t(S) \vcentcolon= \lvert (\bigcup_{i\in u_t} S_i )\cap S\rvert$ is a supermodular function, for all $t \in \mathbb{N}.$
\begin{proof}
Let $S \subseteq T \subset N, e\in N\setminus T, t\in \mathbb{N}.$ We have:
\begin{equation}
\begin{split}
&g_t(T \cup \{e\}) = \left\lvert \bigcup_{i\in u_t} S_i \cap (T \cup \{e\})\right\rvert \\
&= \left\lvert \bigcup_{i\in u_t} S_i \cap \left((T \cap S) \cup (T \setminus S) \cup \{e\}\right)\right\rvert \\
&= \left\lvert \bigcup_{i\in u_t} S_i \cap \left(S \cup (T \setminus S) \cup \{e\}\right)\right\rvert \\
&= \left\lvert \left(\bigcup_{i\in u_t} S_i \cap (S \cup \{e\})\right) \cup \left(\bigcup_{i\in u_t} S_i \cap (T \setminus S)\right)\right\rvert \\
&= \left\lvert \bigcup_{i\in u_t} S_i \cap (S \cup \{e\})\right\rvert + \left\lvert\bigcup_{i\in u_t} S_i \cap (T \setminus S)\right\rvert \\
&= \ g_t(S \cup \{e\}) + \left\lvert\left(\bigcup_{i\in u_t} S_i\right) \cap (T \setminus S)\right\rvert \\
&\geq g_t(S \cup \{e\}).
\end{split}
\end{equation}
\end{proof}
\end{proposition}

Letting $h_t$ be our new objective function, where $h_t(S_t) \vcentcolon= f_t(S_t) -\mu g_t(S_t)$, where $\mu$ is a regularization parameter, we can clearly see that this new objective function is submodular. This fact follows directly from Definition \ref{Supermodularity}. Since the additive inverse of a supermodular function is submodular, and the nonnegative linear combination of submodular functions is itself a submodular function, $h_t$ becomes a submodular function. This, in turn, means that we can employ any standard method, such as the \textsc{Stochastic Greedy} algorithm to produce fairness-oriented solutions to the sequence $h_t$ of objective functions.

\subsection{Greedy Solutions}
Finding a solution for the best clients that can maximize the utility score of both \eqref{FinalComposition} and \eqref{UnionEq} can be done with the naïve greedy algorithm, due to the submodular nature given by Proposition \ref{SubmodularityProp} and Proposition \ref{SuperSubmodular}\cite{Nemhauser1978AnFunctionsI}.
 
A naïve greedy approach starts at round $k$ with an empty set, $S_k \leftarrow \emptyset$, and adds the element $e \in N \setminus S_k$ which provides the highest marginal gain, $\Delta(e|S_k)$.
\begin{equation}\label{Greedy}
S_{k+1} \leftarrow S_{k} \cup \{\mathop{\arg \max}\limits_{e \in N \backslash S_k}(\Delta(e|S_k))\}.
\end{equation}

Nonetheless, if the cardinality of the ground set $|N|$ is too large, searching through this space for the desired elements may prove to be too expensive. Because of this, stochastic variants of the naïve greedy algorithm can be employed \cite{Mirzasoleiman2014LazierGreedy,Hashemi2018ANetworks,Hashemi2021RandomizedSubmodularity,Hashemi2022OnLearning, Hashemi2022OnMaximization} to effectively reduce this search cost while still maintaining high-confidence in the solution provided. This is done by randomly sampling a smaller subset $R$ of size $r$ and searching through this reduced space. That is:
\begin{equation}\label{StochasticGreedy}
S_{k+1} \leftarrow S_k \cup \{\mathop{\arg \max}\limits_{e \in R \backslash S_k}(\Delta(e|S_k))\}.
\end{equation}
Applying the selection strategy utilizing \eqref{FinalComposition} or alternatively \eqref{UnionEq}, and using the \textsc{FedAvg} algorithm as the core method for aggregating client model updates results in our proposed algorithms \textsc{SubTrunc} and \textsc{UnionFL}, summarized as Algorithm \ref{alg:alg1}.

\begin{algorithm}[t]
\algrenewcommand\algorithmicrequire{\textbf{Input:}}
\algrenewcommand\algorithmicensure{\textbf{Output:}}
\caption{\textsc{Fairness-Aware FL Algorithms}}\label{alg:alg1}
\begin{algorithmic}[1]
\Require Truncation, regularization parameters $b, \lambda, \mu \in \mathbb{R}$, and window schedule $u_t$, communication rounds $K$, local steps $E$, participating clients $\kappa$, initial weight vector $w_0$, learning rate $\eta$ 
\Ensure $w_k$, weights for trained model\\
Server initializes $w_0$
\For{$k = {1, \dots, K}$}
\State Subset $S_k$ of size $\kappa$ is selected by the server via the 
\State stochastic variant of the naïve greedy algorithm, 
\State following the formulation of \eqref{FinalComposition} or alternatively the \State formulation of \eqref{UnionEq}. 
\For{client $i \in S_k$}
    \State $w^{(i)}_{k,0} \leftarrow w_k$
    \For{$\tau = 1, \dots, E$}
        \State Select random batch from client $i$: \State$\mathcal{B}_{k,\tau}^{(i)}$ compute stochastic gradient $\tilde{f}_i$ at $w_{k,\tau}^{(i)}$ \State over $\mathcal{B}_{k,\tau}^{(i)}$
        \State{$w_{k,\tau+1}^{(i)} \leftarrow w_{k,\tau}^{(i)} - \eta \nabla \tilde{f}_i(w_{k,\tau}^{(i)};\mathcal{B}_{k,\tau}^{(i)})$}
    \EndFor
    \State{$w_k^{(i)} \leftarrow w_k - w_{k,E}^{(i)}$}
\EndFor
\State $w_{k+1} \leftarrow w_{k} - \frac{1}{\kappa} \sum_{i \in S_i} w_{k}^{(i)}$
\EndFor\\
\textbf{return} $w_K$ \Comment{Final model weights}
\end{algorithmic}
\end{algorithm}

\section{{Theoretical Convergence Analysis}}

\begin{table*}[h]
  \caption{Comparison of Client Selection Methods on non-iid MNIST and CIFAR10.}
  \centering
  \begin{tabular}{@{}ccccccc@{}}
    \toprule
    &&\textbf{MNIST}&&&\textbf{CIFAR10}&\\
    \bottomrule
    \toprule    
    Method  & Training Loss & Testing Acc [\%] & Client Dissimilarity [\%] & Training Loss & Testing Acc [\%] & Client Dissimilarity [\%] \\
    \midrule
    \textsc{DivFL} & $\bm{0.16 \pm 0.02}$ & $82.16 \pm 0.91$ & $8.89 \pm 0.59$ & $\bm{0.88 \pm 0.03}$ & $35.40 \pm 1.28$ & $12.62 \pm 0.95$\\
    \textsc{SubTrunc} & $0.21 \pm 0.01$ & $\bm{83.72 \pm 0.81}$& $\bm{7.96 \pm 0.62}$ & $0.89 \pm 0.03$ & $35.53 \pm 1.33$& $\bm{12.46 \pm 0.94}$\\
    \textsc{UnionFL} & $0.17 \pm 0.01$ & $83.57 \pm 0.34$ & $8.59 \pm 0.34$ & $0.88 \pm 0.03$ & $\bm{35.63 \pm 1.09}$ & $12.55 \pm 0.79$ \\
    \textsc{Random} & $0.18 \pm 0.02$ & $82.99 \pm 0.99$ & $9.16 \pm 0.18$ & $0.92 \pm 0.01$ & $32.55 \pm 0.94$ & $14.42 \pm 0.18$\\
    \textsc{Power-of-Choice} & $0.21 \pm 0.01$ & $82.42 \pm 1.97$ & $8.59 \pm 0.38$ & $1.00 \pm 0.01$ & $30.17 \pm 0.65$ & $13.82 \pm 0.19$\\
    \bottomrule
  \end{tabular}
  \label{table:MainResults}
  \vspace{-3mm}
\end{table*}
In this section, we analyze the convergence properties of the proposed algorithms under standard assumptions in nonconvex FL. Our analysis utilizes more relaxed assumptions compared to \cite{BalakrishnanDiverseMaximization}, complementing our proposed fairness-promoting formulation.
\subsection{{Assumptions and Related Concepts}}

The definitions and assumptions used to analyze the performance of the algorithm are listed below. These are standard and ubiquitous in the analysis of training algorithms for nonconvex machine learning and FL problems \cite{RudrajitDas2022FasterMomentum,Li2019OnData,Karimireddy2019SCAFFOLD:Learning,Stich2018LocalLittle,Acharya2021RobustDescent}.

\begin{definition}[$L$-smoothness]
A function $f: \Omega \rightarrow \mathbb{R}$ is considered to be $L$-smooth if $\forall\, w, w' \in \Omega$, $\lVert\nabla f(w) - \nabla f(w')\rvert \leq L\lVert w-w'\rVert$. Additionally, if $f$ can be differentiated twice, then $\forall\, w, w' \in \Omega$, $f(w')\leq f(w) + \langle \nabla f(w), w' - w\rangle + \frac{L}{2}\lVert w' - w\rVert^2$.
\end{definition}

\begin{assumption}[$L$-smoothness]\label{Smoothness_Assumption}For all $x$, we assume $l(x,w)$ to be $L$-smooth with respect to $w$. Then, each $f_i(w)$ defined in \eqref{DataDist} where $i \in [n]$ is $L$-smooth, and so is $f(w)$ defined in \eqref{FedAvg}.
\end{assumption}

\begin{assumption}[Nonnegativity]\label{Non-negativity_Assumption}
Each $f_i(w)$ is non-negative, therefore, $f_i^* := \min_w f_i(w) \geq 0$.
\end{assumption}
The nonnegativity assumption holds without loss of generality as any function bounded from below can be shifted to satisfy this assumption. Furthermore, almost all loss functions of interest in FL are nonnegative.

\begin{assumption}[Bounded bias]\label{BoundedBias}
Let $\tau$ be the local update steps, with $\tau \in \{0, \dots, E - 1\}$, K be the communication rounds with $k \in \{0, ..., K-1\}$, $\hat{u}_{k,\tau}^{\left(i\right)}$ be the stochastic gradient of client $i$ at communication round $k$ and local update step $\tau$ and let $b_{k,\tau}$ be the approximation error for the true gradient at communication round $k$ and local update step $\tau$. We assume that the gradient of the subset constructed by our objective functions\eqref{FinalComposition} or \eqref{UnionEq} is the gradient of the global loss function plus an approximation error, $b_{k,\tau}$ which we refer to as bias. That is:
\begin{equation}\label{BoundedBiasEq}
\frac{1}{\kappa}\sum_{i\in{S_k}}\hat{u}_{k,\tau}^{\left(i\right)} = \hat{u}_{k,\tau} + b_{k,\tau}.
\end{equation}

We assume that this bias, $b_{k,\tau}$ is bounded by: $b_{k,\tau} \leq \gamma$.
\end{assumption}

Given that we are selecting a subset of total clients, this assumption allows us to approximate the true gradient up to a certain approximation error $b_{k,\tau}$. This assumption is also used in prior work \cite{BalakrishnanDiverseMaximization,Mirzasoleiman2019CoresetsModels}. Similar to \cite{BalakrishnanDiverseMaximization}, we assume that the gradient of the subset constructed by our objective function approximates the full gradient of the overall set, with the addition of an approximation error which we model as bias.

\subsection{{Main Theoretical Results}}
Theorem \ref{Theorem1} states the convergence properties of the proposed \textsc{SubTrunc} and \textsc{UnionFL} methods. 
\begin{theorem}
\label{Theorem1}
Let Assumptions \ref{Smoothness_Assumption}, \ref{Non-negativity_Assumption}, and \ref{BoundedBias} hold. Set $\eta_k = \frac{1}{LE}\sqrt{\frac{1}{K}}$$\, \forall\, k \in \{0, ..., K-1\}$. Let $\mathbb{P}$ be a distribution such that $\mathbb{P}\left(k\right) = \frac{\left(1+\zeta\right)^{\left(K-1-k\right)}}{\sum_{k'=0}^{K-1}\left(1+\zeta\right)^{\left(K-1-k'\right)}}$, where $\zeta := \eta^2L^2E^2\left(\frac{9\eta LE}{4}\right)$. Let $k^* \sim \mathbb{P}$. Then, for $K \geq 9$:
\begin{equation}\label{Convergence}
\begin{gathered}
\E\left[\|\nabla f\left(w_{k^*}\right)\|^2\right] \leq \frac{12Lf\left(w_0\right)}{\sqrt{K}} + \\
\left(\frac{2}{EnK} + \frac{3}{K} + \frac{4}{En\sqrt{K}}\right)\sigma^2 + \left(4+ \frac{4}{\sqrt{K}}\right)\gamma.
\end{gathered}
\end{equation}
That is, there exists a learning rate and a nonuniform distribution on the iterates such that if the output is generated according to that distribution, the expected performance satisfies:
\begin{equation}
    \E\left[\|\nabla f\left(w_{k^*}\right)\|^2\right] = \mathcal{O}\left(\frac{1}{\sqrt{K}}+\frac{\sigma^2}{En\sqrt{K}}+\gamma\right).
\end{equation}
\end{theorem}

The complete proof can be found in Appendix \ref{appendix:FullProof}. Theorem \ref{Theorem1} establishes a bound on the so-called approximate first-order stationary point of the global model parameters, i.e., the condition that  $\E{\left[\|f\left(w_{k^*}\right)\|^2\right]} \leq \epsilon$ . In particular, we can see that the convergence error consists of three terms: the first term quantifies the impact of the initialization. The second term captures the impact of the statistical noise in the local stochastic gradients utilized by the clients. Finally, the last term captures the impact of bias that arises from using the proposed client selection strategy, where the bias is defined in Assumption \ref{BoundedBias}. Theorem \ref{Theorem1}, hence, indicates that as long as $K = \Omega(1/\epsilon^2)$, it holds that $\E\left[\|\nabla f\left(w_{k^*}\right)\|^2\right]\leq \mathcal{O}(\epsilon+\gamma)$. Thus, if the bias parameter satisfies $\gamma = \mathcal{O}(\epsilon)$, our algorithms can identify an $\epsilon$-accurate first-order stationary solution.

Unlike prior work, we analyze the proposed method without making the Bounded Client Dissimilarity assumption, by leveraging the smoothness of the clients' loss functions and a nonuniform sampling technique to output a solution. Additionally, compared to \cite{BalakrishnanDiverseMaximization} which assumes the restrictive assumption of strong convexity, our analysis applies to general nonconvex problems that arise frequently in practical and large-scale application of FL in the big data setting.

Finally, note that adopting a distribution over the iterates to output a global model is a standard approach to state the theoretical convergence results for FL and optimization algorithms \cite{Reisizadeh2019FedPAQ:Quantization,RudrajitDas2022FasterMomentum,Karimireddy2019SCAFFOLD:Learning}. In practice, however, the latest global model is used for inference.


\section{{Experimental Results}}\label{exp}

\begin{table*}[h]
  \caption{Effects of a varying $\lambda$ under a non-iid setting within MNIST and CIFAR10.}
  \centering
  \begin{tabular}{@{}ccccc@{}}
    \toprule
    &\textbf{MNIST}&&\textbf{CIFAR10}&\\
    \bottomrule
    \toprule
    $\lambda$-Value  & Testing Accuracy [\%] & Client Dissimilarity [\%] & Testing Accuracy [\%] & Client Dissimilarity [\%]\\
    \midrule
    $0.01$ & $82.18 \pm 0.99$ & $8.90 \pm 0.63$ & $35.40 \pm 1.28$ & $12.62 \pm 0.95$\\
    $0.10$ & $83.13 \pm 1.04$ & $8.38 \pm 0.56$ & $35.47 \pm 1.33$ & $12.58 \pm 0.92$\\
    $0.25$ & $83.60 \pm 0.78$ & $8.18 \pm 0.55$ & $35.47 \pm 1.29$ & $12.47 \pm 0.96$\\
    $0.50$ & $83.76 \pm 0.42$ & $8.14 \pm 0.34$ & $35.53 \pm 1.33$ & $12.46 \pm 0.94$\\
    $0.75$ & $83.94 \pm 0.33$ & $8.07 \pm 0.41$ & $35.55 \pm 1.32$ & $12.46 \pm 1.08$\\
    $0.95$ & $83.72 \pm 0.81$ & $7.96 \pm 0.62$ & $35.50 \pm 1.33$ & $12.45 \pm 1.05$\\
    \bottomrule
  \end{tabular}
  \label{table:LambdaVar}
    \vspace{-3mm}
\end{table*}


The performance of both \textsc{UnionFL} and \textsc{SubTrunc} is evaluated on the MNIST dataset of handwritten digits as well as on the CIFAR10 dataset, both under a non-iid data distribution. We benchmark these algorithm's performance to \textsc{DivFL}, \textsc{Power-of-Choice} as well as the random sampling of clients without replacement which is a simple but standard benchmark for client sampling. The experiments are simulated on a pool of $\lvert N\rvert = 100$ clients, but in order to avoid scanning through the whole set, we employ a stochastic greedy search with a subset $R$ of $r = 10$ clients with $\kappa = 10$, on the LeNet architecture.
 
Both datasets are partitioned in such a way that each client had 3 equally partitioned distinct classes. Unless otherwise stated, the truncation factor is set at a value of $b = 1.10$ and $\mu = 1$, which seemed to be the best values that resulted in the lowest Client Dissimilarity scores for these particular datasets.


Our evaluation employs three key metrics: training loss, test accuracy, and client dissimilarity. The training loss metric assesses the convergence behavior, while the test accuracy serves to gauge its generalization capability. The client dissimilarity metric characterizes the balance and fairness of the global model.

Client dissimilarity is calculated by taking the difference of the final model's performance on the client's test dataset, across all clients. Measuring model dissimilarity this way allows us to better express the divergence in model performance across different clients.
\begin{table*}[h]
  \caption{Effects of a varying window size under a non-iid setting within MNIST and CIFAR10.}
  \centering
  \begin{tabular}{@{}ccccc@{}}
    \toprule
    &\textbf{MNIST}&&\textbf{CIFAR10}&\\
    \bottomrule
    \toprule
    Window Size  & Testing Accuracy [\%] & Client Dissimilarity [\%] & Testing Accuracy [\%] & Client Dissimilarity [\%]\\
    \midrule
    $2$ & $83.27 \pm 0.54$ & $8.77 \pm 0.35$ & $35.62 \pm 0.99$ & $12.51 \pm 0.68$\\
    $5$ & $83.57 \pm 0.34$ & $8.59 \pm 0.34$ & $35.63 \pm 1.09$ & $12.55 \pm 0.79$\\
    $10$ & $83.74 \pm 0.36$ & $8.52 \pm 0.18$ & $35.86 \pm 1.06$ & $12.56 \pm 0.88$\\
    \bottomrule
  \end{tabular}
  \label{table:WindowVar}
    \vspace{-3mm}
\end{table*}

\begin{table*}[h]
  \caption{Comparison of $\phi$ choice under a non-iid setting within MNIST and CIFAR10.}
  \centering
  \begin{tabular}{@{}cccccc@{}}
    \toprule
    &\textbf{MNIST}&&\textbf{CIFAR10}&\\
    \bottomrule
    \toprule
    $\phi\left(\cdot\right)$ & Testing Accuracy [\%] & Client Dissimilarity [\%] & Testing Accuracy [\%] & Client Dissimilarity [\%]\\
    \midrule
    Identity  &$83.42 \pm 0.12$ & $7.91 \pm 0.35$ & $35.29 \pm 0.98$ & $11.93 \pm 0.70$\\
    $ln\left(1 + x\right)$ & $83.72 \pm 0.81$ & $7.96 \pm 0.62$ & $35.53 \pm 1.33$ & $12.46 \pm 0.94$\\
    \bottomrule
  \end{tabular}
  \label{table:PhiComparison}
    \vspace{-3mm}
\end{table*}

\subsection{{Results on Non-IID MNIST \& CIFAR10}}

Under both non-iid scenarios, \textsc{SubTrunc} and \textsc{UnionFL} outperform the baselines by achieving a lower overall client dissimilarity score, indicating that the final model's performance is more consistent across clients, with \textsc{SubTrunc} being able to consistently achieve the lowest score among all methods. Additionally both \textsc{SubTrunc} and \textsc{UnionFL} converge at a similar rate to the baseline methods as evidenced by their training loss while maintaining a comparable performance under testing accuracy. These results can be seen in Table \ref{table:MainResults}, where the bolded entries represent the best performance.

 
The dissimilarity evolution of our method demonstrates that our algorithm \textsc{SubTrunc} is effective in ensuring a more balanced and thus fair performance across all clients even under the presence of high data heterogeneity throughout the course of training.


\subsection{{Tuning Lambda \& Window Size}}

Seeking to better understand the effects of $\lambda$ on our fairness-aware term and the window size $u_t$ on our union-based objective, we simulate results under the same non-iid conditions and setting as described at the beginning of Section \ref{exp}, while varying $\lambda$ for \textsc{SubTrunc} and window-size  $u_t$ for \textsc{UnionFL}.
 
Table \ref{table:LambdaVar} shows that as the values for $\lambda$ get closer to zero, i.e. as the method approaches \textsc{DivFL}, the client dissimilarity score degrades, whereas as $\lambda$ increases, the client dissimilarity improves, while still maintaining a comparable testing accuracy across the gamut of $\lambda$ values, 
highlighting the important trade-off between performance-centric models versus balanced models.

Table \ref{table:WindowVar} highlights the effect that window-size has on both the final model's accuracy as well as the divergence of the final model's performance across clients. Increasing the window size results in model's that perform better under both a testing accuracy and client dissimilarity standpoint. This is to be expected as increasing this window size, leads to more diversified solutions that would otherwise not have been considered.


\subsection{The Effect of $\phi$}

For the results shown on tables \ref{table:MainResults} and \ref{table:LambdaVar}, $\phi$ was defined as $\phi(x) = \ln\left(1 + x\right)$, where $\ln\left(1 + x\right)$ is a function with a Lipschitz constant of $L < 1$. In seeking to understand the effects of a different choice of  $\phi$ within \textsc{SubTrunc}, we leveraged a function with a similar Lipschitz constant to $\ln\left(1 + x\right)$, the identity function. This new function was able to produce lower client dissimilarity scores across both datasets under non-iid conditions and demonstrated the enhancement effect some functions can exhibit when employed within our algorithm.
Table \ref{table:PhiComparison} shows these results.

\section{{Conclusion}}


In this paper, we introduced a novel approach to addressing fairness concerns within the Federated Learning (FL) framework by incorporating a fairness-aware term into the submodular maximization method for solving the client selection problem. This led to the development of our algorithms, \textsc{SubTrunc} and \textsc{UnionFL}. Our proposed algorithms enable the derivation of models that exhibit a more balanced performance across clients, shedding light on the nuanced interplay between accuracy and model uniformity through the careful tuning of parameters such as $b$ (the truncation parameter), $\lambda$ (the weighting of the fairness-aware term), and $u_t$ (the window size).

The primary advantage of our methods lies in their ability to reduce client dissimilarity, thereby catering to practical applications across several domains, where achieving consistent and accurate performance across highly diverse clients is paramount.

Furthermore, we established theoretically that \textsc{SubTrunc} and \textsc{UnionFL} require $K = \mathcal{O}(1/\epsilon^2)$ rounds of communication to attain $\E{\left[\|f\left(w_{k^*}\right)\|^2\right]} \leq \epsilon$, doing so under notably milder assumptions compared to prior methodologies. Our experimental findings corroborated the efficacy of our algorithms, revealing that they yield more balanced models compared to both random selection strategies and alternatives like \textsc{DivFL} and \textsc{Power-of-Choice}, as assessed through a client dissimilarity metric.

The incorporation of this fairness-aware term renders our algorithm a straightforward and practical solution, capable of achieving a more balanced and equitable model without compromising performance in federated settings.

As part of future work, it would be valuable to consider extensions such as incorporating robustness considerations \cite{kaya2024randomized,kaya2024localized}, providing high probability guarantees \cite{kaya2023high}, and dealing with the possibility of imperfect gradient availability \cite{upadhyay2023improved,kaya2023communication}.


\section*{Acknowledgments}

The authors would like to thank Dr. Vijay Gupta (Elmore Family School of Electrical and Computer Engineering, Purdue University) for the fruitful discussion regarding this work. This work was supported in part by NSF CNS 2313109.

\bibliographystyle{ieeetr}
\bibliography{refs-cdc-ah.bib}




\appendix
\subsection{{Convergence Analysis for \textsc{SubTrunc} and \textsc{UnionFL}}} \label{appendix:FullProof}

We will provide convergence results for \textsc{SubTrunc}, without the Bounded Client Dissimilarity Assumption. Following a similar outline to \cite{RudrajitDas2022FasterMomentum}, we use Lemma \ref{EpsilonBound} and exploit the $L$-smoothness of $f$ to obtain a bound on the per-round progress of the algorithm. Additionally both Lemmas \ref{LemmaNonBCD},
\ref{IndGrads} are used to help bound terms deriving from the analysis of Lemma \ref{EpsilonBound}, mainly dealing with bounding the gradient error at any given round as well as providing a bound on the expected value of the gradient for each client at any given communication round $k$ and local step $\tau$. Different from \cite{RudrajitDas2022FasterMomentum}, we carefully account for the bias terms that arise from utilizing the submodular client selection method by leveraging Assumption \ref{BoundedBias}.

Now, we present the proof of Theorem \ref{Theorem1} which states the convergence properties of \textsc{SubTrunc}.

\begin{proof}
Using the results from Lemma \ref{EpsilonBound}, for $\eta_kLE \leq \frac{1}{3}$, the per-round progress can be bounded as follows:
\begin{equation*}\label{FinalLemmaBound}
\begin{gathered}
\E\left[f\left(w_{k+1}\right)\right] \leq \E \left[f\left(w_k\right)\right] - \frac{\eta_kE}{4}\E\left[\|\nabla f\left(w_k\right)\|^2\right] + \\
\eta_k^2LE^2\left(\frac{9\eta_kLE}{4}\right)\frac{1}{n}\sum_{i\in\left[n\right]}\E\left[\|\nabla f_i\left(w_k\right)\|^2\right]
+ \\
\eta_k^2LE\left(\frac{\eta_kLE}{2}\left(\frac{1}{n} + \frac{3E}{4}\right)+ \frac{1}{n}\right)\sigma^2 + \eta_k\left(1 + \eta_kLE\right)E\gamma.
\end{gathered}
\end{equation*}

Now, by utilizing $f_i$'s attributes of $L$-smoothness and non-negativity, we obtain the following:

\begin{equation*}\label{LTrick}
\begin{gathered}
\sum_{i\in\left[n\right]}\|\nabla f_i\left(w_k\right)\|^2 \leq \sum_{i\in\left[n\right]} 2L\left(f_i\left(w_k\right) - f_i^*\right)\\
\leq 2nLf\left(w_k\right) - 2L \leq \sum_{i\in\left[n\right]}f_i^*\leq 2nLf\left(w_k\right).
\end{gathered}
\end{equation*}

Plugging the above to the result of Lemma \ref{EpsilonBound} and setting a constant learning rate $\eta_k = \eta$, we get:
\begin{equation}\label{SimplifiedFinal}
\begin{gathered}
\E\left[f\left(w_{k+1}\right)\right] \leq \left(1 + \eta^2L^2E^2\left(\frac{9\eta LE}{4}\right)\right)\E\left[f\left(w_k\right)\right] - \\
\frac{\eta E}{4}\E\left[\|\nabla f\left(w_k\right)\|^2\right]
+ \eta^2LE\left(\frac{\eta LE}{2}\left(\frac{1}{n} + \frac{3E}{4}\right) + \frac{1}{n}\right)\sigma^2 +\\
\eta\left(1 + \eta LE\right)E\gamma.
\end{gathered}
\end{equation}

Let us define the following $\zeta := \eta^2L^2E^2\left(\frac{9\eta LE}{4}\right)$, $\zeta_2:= \left(\frac{\eta LE}{2}\left(\frac{1}{n} + \frac{3E}{4}\right) + \frac{1}{n}\right)$ and finally $\zeta_3:= \left(1 + \eta LE\right)$. Then substituting the above, and unfolding the recursion of \eqref{SimplifiedFinal}, we obtain:

\begin{equation}\label{RecursionUnfold}
\begin{gathered}
\E\left[f\left(w_K\right)\right] \leq \left(1 + \zeta\right)^Kf\left(w_0\right) -\\
\frac{\eta E}{4}\sum_{k=0}^{K-1}\left(1 + \zeta\right)^{\left(K-1-k\right)}\E\left[\|\nabla f\left(w_k\right)\|^2\right] + \\
\eta^2 LE\zeta_2\sigma^2\sum_{k=0}^{K-1}\left(1 + \zeta\right)^{\left(K-1-k\right)} +
\eta\zeta_3 E\gamma\sum_{k=0}^{K-1}\left(1-\zeta\right)^{\left(K-1-k\right)}.
\end{gathered}
\end{equation}

Define $\mathbb{P}\left(k\right):= \frac{\left(1-\zeta\right)^{\left(K-1-k\right)}}{\sum_{k'=0}^{K-1}\left(1+\zeta\right)^{\left(K-1-k'\right)}}$. Then, by re-arranging equation \eqref{RecursionUnfold} and using $\E\left[f\left(w_K\right)\right] \geq 0$, we get the following:

\begin{equation}\label{ProbInBound}
\begin{gathered}
\sum_{k=0}^{K-1}p_k\E\left[\|\nabla f\left(w_k\right)\|^2\right] \leq \frac{4\left(1 + \zeta\right)^Kf\left(w_0\right)}{\eta E \sum_{k'=0}^{K-1}\left(1+\zeta\right)^{k'}} + \\
4\eta L\zeta_2\sigma^2 + 4\zeta_3 \gamma
\end{gathered}
\end{equation}

\begin{equation}\label{ProbInBound2}
 = \frac{4\zeta f\left(w_0\right)}{\eta E \left(1 - \left(1 + \zeta\right)^{-K}\right)} + 4\eta L\zeta_2\sigma^2 + 4\zeta_3\gamma.
\end{equation}

This last step follows from $\sum_{k'=0}^{K-1}\left(1+\zeta\right)^{k'} = \frac{\left(1 + \zeta\right)^K - 1}{\zeta}$. Now, let us show:

\begin{equation*}\label{ZetaBound}
\begin{gathered}
\left(1 + \zeta\right)^{-K} < 1 - \zeta K + \zeta^2\frac{K\left(K + 1\right)}{2} < 1 - \zeta K + \zeta^2K^2 \\
\implies 1 - \left( 1 + \zeta\right)^{-K} > \zeta K\left(1 - \zeta K\right).
\end{gathered}
\end{equation*}

Using the above on equation \eqref{ProbInBound2}, we have that for $\zeta K < 1$:

\begin{equation}\label{FinalBeforeValue}
\begin{gathered}
\sum_{k=0}^{K-1}p_k\E\left[\|\nabla f\left(w_k\right)\|^2\right] \leq \frac{4f\left(w_0\right)}{\eta EK \left(1 - \zeta K\right)} + \\
4\eta L\left(\frac{\eta LE}{2} \left(\frac{1}{n} + \frac{3E}{4}\right) + \frac{1}{n} \right)\sigma^2 + 4\left(1 + \eta LE\right)\gamma.
\end{gathered}
\end{equation}

Now, let us pick $\eta = \frac{1}{LE}\sqrt{\frac{1}{K}}$. We need to make sure that $\eta LE \leq \frac{1}{3}$; this happens for $K \geq 9$. Plugging $\eta LE = \frac{1}{LE}\sqrt{\frac{1}{K}}$ in \eqref{FinalBeforeValue}, and using $1 - \zeta K \geq \frac{1}{3}$, we get:

\begin{equation}\label{FinalResult}
\begin{gathered}
\sum_{k=0}^{K-1}p_k\E\left[\|\nabla f\left(w_k\right)\|^2\right] \leq \frac{12 L f\left(w_0\right)}{\sqrt{K}} + \\
\left(\frac{2}{EnK} + \frac{3}{K} + \frac{4}{En\sqrt{K}}\right)\sigma^2 + \left(4 + \frac{4}{\sqrt{K}}\right)\gamma.
\end{gathered}
\end{equation}

This finishes the proof.
\end{proof}

Next, we provide the intermediate lemmas used in the proof of Theorem \ref{Theorem1}.

Using the following Lemma \ref{EpsilonBound} we will be able to bound the per-round progress of the algorithm. 
\begin{lemma}\label{EpsilonBound} With $\eta_kLE \leq \frac{1}{3}$, we have:

\begin{equation}\label{FinalLemma13}
\begin{gathered}
\E\left[f\left(w_{k+1}\right)\right] \leq \E \left[f\left(w_k\right)\right] - \frac{\eta_kE}{4}\E\left[\|\nabla f\left(w_k\right)\|^2\right] + \\ \eta_k^2LE^2\left(\frac{9\eta_kLE}{4}\right)\frac{1}{n}\sum_{i\in\left[n\right]}\E\left[\|\nabla f_i\left(w_k\right)\|^2\right] \\
+ \eta_k^2LE\left(\frac{\eta_kLE}{2}\left(\frac{1}{n} + \frac{3E}{4}\right)+ \frac{1}{n}\right)\sigma^2 +\\
\eta_k\left(1 + \eta_kLE\right)E\gamma.
\end{gathered}
\end{equation}
\end{lemma}

\begin{proof} 
Define:
\begin{equation*}\label{ProofPrelims}
\begin{gathered}
    \hat{u}_{k,\tau}^{\left(i\right)} := \nabla\tilde{f}_i \left(w_{k,\tau}^{\left(i\right)};\mathcal{B}_{k,\tau}^{\left(i\right)}\right),\,\\
    \hat{u}_{k,\tau} := \frac{1}{n}\sum_{i\in\left[n\right]}\hat{u}_{k,\tau}^{\left(i\right)},\\
    u_{k,\tau} := \frac{1}{n} 
    \sum_{i\in\left[n\right]}\nabla f_i 
    \left(w_{k,\tau}^{\left(i\right)}\right),
\end{gathered}
\end{equation*}

\begin{equation*}\label{ProofPrelims2}
\begin{gathered}
 \frac{1}{\kappa}\sum_{i\in{S_k}}\hat{u}_{k,\tau}^{\left(i\right)} = \hat{u}_{k,\tau} + b_{k,\tau},\, \\
 \bar{w}_{k,\tau} := \frac{1}{n}\sum_{i\in\left[n\right]}w_{k,\tau}^{\left(i\right)}, \, \tilde{e}_{k,\tau}^{\left(i\right)} = \nabla f_i \left(w_{k,\tau}^{\left(i\right)}\right) - \nabla f_i\left(\bar{w}_{k,\tau}\right).
\end{gathered}
\end{equation*}

Then:

\begin{equation}\label{LocalUpdate}
w_{k+1} = w_k - \eta_k \sum_{\tau=0}^{E-1} \left(\frac{1}{r}\sum_{i\in{S_k}}\hat{u}_{k,\tau}^{\left(i\right)}\right),
\end{equation}

\begin{equation}\label{GlobalUpdate}
\bar{w}_{k,\tau} = w_k - \eta_k \sum_{t=0}^{\tau - 1}\hat{u}_{k,t},
\end{equation}

\begin{equation}\label{ExpectationOfGrad}
    \E_{\{\mathcal{B}_{k,\tau}^{\left(i\right)}\}_{i=1}^{n}}\left[\hat{u}_{k,\tau}\right] = u_{k,\tau},
\end{equation}

\begin{equation}\label{OverallGradBound}
    \E{\left[\|\sum_{t=0}^{\tau-1} \hat{u}_{k,t}\|^2\right]} \leq \tau \sum_{t=0}^{\tau - 1}\E{\left[\|u_{k,t}\|^2\right]} + \frac{\tau \sigma^2}{n},
\end{equation}

\begin{equation}\label{ClientGradBound}
    \E{\left[\|\sum_{t=0}^{\tau - 1}\hat{u}_{k,t}^{\left(i\right)}\|^2\right]} \leq \tau \sum_{t=0}^{\tau - 1} \E{\left[\|\nabla f_i\left(w_{k,t}^{\left(i\right)}\right)\|^2\right]} + \tau \sigma^2,
\end{equation}

\begin{equation}\label{BiasBound}
\sum_{\tau=0}^{E-1}\E\left[\|b_{k,\tau}\|^2\right] \leq E\gamma.
\end{equation}

Now, by using $f$'s $L$-smoothness and \eqref{LocalUpdate}, we obtain the following:

\begin{equation}\label{Initial}
\begin{gathered}
\E{\left[f\left(w_{k+1}\right)\right]} \leq \E{\left[f\left(w_k\right)\right]} - \\
\E{\left[\left<\nabla f\left(w_k\right), \eta_k \sum_{\tau=0}^{E-1} \left(\frac{1}{\kappa}\sum_{i\in {S_k}}\hat{u}_{k,\tau}^{\left(i\right)}\right)\right>\right]} + \\
\frac{L}{2}\E{\left[\|\eta_k \sum_{\tau=0}^{E-1} \left(\frac{1}{\kappa}\sum_{i\in {S_k}}\hat{u}_{k,\tau}^{\left(i\right)}\right) \|^2\right]},
\end{gathered}
\end{equation}

\begin{equation}\label{Replacement}
\begin{gathered}
= \E{\left[f\left(w_k\right)\right]} - \E{\left[\left<\nabla f\left(w_k\right), \sum_{\tau = 0}^{E-1}\eta_k \left(\hat{u}_{k,\tau} + b_{k,\tau}\right)\right>\right]} +\\
\frac{\eta_k^2 L}{2}\E{\left[\|\sum_{\tau=0}^{E-1}\hat{u}_{k,\tau} + b_{k,\tau}\|^2\right]},
\end{gathered}
\end{equation}

\begin{equation}\label{SeperateInner}
\begin{gathered}
= \E{\left[f\left(w_k\right)\right]} - \\
\E{\left[\underbrace{\left<\nabla f\left(w_k\right), \sum_{\tau = 0}^{E-1}\eta_k u_{k,\tau}\right>}_\text{$A$} + \underbrace{\left<\nabla f\left(w_k\right), \sum_{\tau = 0}^{E-1}\eta_k b_{k,\tau}\right>}_\text{$B$}\right]} + \\
\frac{\eta_k^2 L}{2}\underbrace{\E{\left[\|\sum_{\tau=0}^{E-1}\hat{u}_{k,\tau} + b_{k,\tau}\|^2\right]}}_\text{$C$}.
\end{gathered}
\end{equation}

We can now use that for 2 vectors \textbf{$x$}, \textbf{$y$}, the following holds: $\left<x,y\right> = \frac{1}{2}\left(\|x\|^2 + \|y\|^2 - \|x - y\|^2\right)$. Leveraging this for term $A$, we obtain the following:

\begin{equation}\label{TermADecomp}
\begin{gathered}
    \left<\nabla f\left(w_k\right), \sum_{\tau = 0}^{E-1} u_{k,\tau}\right> = \sum_{\tau =0}^{E-1} \left<\nabla f\left(w_k\right), u_{k,\tau}\right> = \\
    \frac{1}{2}\sum_{\tau = 0}^{E-1}\left(\|\nabla f\left(w_k\right)\|^2 + \|u_{k,\tau}\|^2 - \|\nabla f \left(w_k\right) - u_{k,\tau}\|^2\right).
\end{gathered}
\end{equation}

Additionally, we also have that for 2 vectors, \textbf{$x$}, \textbf{$y$}, we have that $\left<x,y\right> \leq \frac{\lambda}{2} \|x\|^2 + \frac{1}{2\lambda}\|x\|^2$. Using $\lambda = \frac{1}{2}$ on $B$, we obtain:

\begin{equation*}\label{TermBDecomp}
\begin{gathered}
    - \left<\nabla f\left(w_k\right), \sum_{\tau = 0}^{E-1}\eta_k b_{k,\tau}\right> = \eta_k \sum_{\tau=0}^{E-1}\left<\nabla f \left(w_k\right), -b_{k,\tau}\right> \\
    \leq \eta_k\sum_{\tau=0}^{E-1}\left(\frac{1}{4}\|\nabla f \left(w_k\right)\|^2 + \|b_{k,\tau}\|^2\right),
\end{gathered}
\end{equation*}

\begin{equation}\label{TermBDecomp2}
\leq \frac{\eta_k E}{4}\|\nabla f \left(w_k\right)\|^2 + \eta_k \sum_{\tau=0}^{E-1} \|b_{k,\tau}\|^2.
\end{equation}

Now, we have the following: $\E\left[\|x + y\|^2\right] \leq 2\E\left[\|x\|^2\right] + 2\E\left[\|y\|^2\right]$ and $\E\left[\|\sum_{i \in R}x_i\|^2\right] \leq |R|\sum_{i\in R} \E\left[\|x_i\|^2\right]$. Using these as well as \eqref{OverallGradBound} on term $C$, we obtain the following:

\begin{equation}\label{TermCDecomp}
\begin{gathered}
\E{\left[\|\sum_{\tau=0}^{E-1}\hat{u}_{k,\tau} + b_{k,\tau}\|^2\right]} \leq 2\E\left[\|\sum_{\tau=0}^{E-1}\hat{u}_{k,\tau}\|^2\right] + \\
2\E\left[\|\sum_{\tau=0}^{E-1}b_{k,\tau}\|^2\right],
\end{gathered}
\end{equation}

\begin{equation}\label{TermCDecomp2}
\leq 2E\sum_{\tau=0}^{E-1}\E\left[\|u_{k,\tau}\|^2\right] + \frac{2E\sigma^2}{n} + 2E\sum_{\tau=0}^{E-1}\E\left[\|b_{k,\tau}\|^2\right].
\end{equation}

Using the results from \eqref{TermADecomp} \eqref{TermBDecomp2} \eqref{TermCDecomp2} and plugging them on \eqref{SeperateInner}, we obtain the following:

\begin{equation}\label{Intermediary}
\begin{gathered}
    \E\left[f\left(w_{k+1}\right)\right] \leq \E\left[f\left(w_k\right)\right] - \frac{\eta_k E}{4} \E\left[\|\nabla f\left(w_k\right)\|^2\right] + \\
    \underbrace{\frac{\eta_k}{2}\sum_{\tau=0}^{E-1}\E\left[\|\nabla f\left(w_k\right) - u_{k,\tau}\|^2\right]}_\text{$D$} - \frac{\eta_k}{2}\left(1 - \eta_k LE\right)\sum_{\tau=0}^{E-1}\E\left[\|u_{k,\tau}\|^2\right]\\
    + \eta_k \left(1 + \eta_k LE\right)\sum_{\tau=0}^{E-1}\E\left[\|b_{k,\tau}\|^2\right] + \frac{\eta_k^2LE\sigma^2}{n}.
\end{gathered}
\end{equation}

Now, term $D$ is upperbounded by Lemma \ref{LemmaNonBCD}, plugging this result on equation \eqref{Intermediary} and using \eqref{BiasBound}, we get:
\begin{equation}\label{AlmostFinal}
\begin{gathered}
\E\left[f\left(w_{k+1}\right)\right] \leq \E\left[f\left(w_k\right)\right] - \frac{\eta_kE}{4}\E\left[\|\nabla f\left(w_k\right)\|^2\right] - \\
\frac{\eta_k}{2}\underbrace{\left(1 - 2\eta_kLE - \eta_k^2L^2E^2\right)}_\text{$E$}\sum_{\tau=0}^{E-1}\E\left[\|u_{k,\tau}\|^2\right] \\
+ \eta_k^2LE^2\left(\frac{9\eta_kLE}{4}\right)\frac{1}{n}\sum_{i\in\left[n\right]}\E\left[\|\nabla f_i\left(w_k\right)\|^2\right]\\
+ \eta_k^2LE\left(\eta_kLE\left(\frac{1}{n} + \frac{3E}{4}\right) + \frac{1}{n}\right)\sigma^2\\
+ \eta_k\left(1 + \eta_k LE\right)\sum_{\tau=0}^{E-1}\E\left[\|b_{k,\tau}\|^2\right].
\end{gathered}
\end{equation}

We can see that term $E \geq 0$ for $\eta_k LE \leq \frac{1}{3}$. Due to the above, we drop this term and for $\eta_kLE \leq \frac{1}{3}$:

\begin{equation}\label{Final}
\begin{gathered}
\E\left[f\left(w_{k+1}\right)\right] \leq \E\left[f\left(w_k\right)\right] - \frac{\eta_kE}{4}\E\left[\|\nabla f\left(w_k\right)\|^2\right] + \\
\eta_k^2LE^2\left(\frac{9\eta_kLE}{4}\right)\frac{1}{n}\sum_{i\in\left[n\right]}\E\left[\|\nabla f_i\left(w_k\right)\|^2\right]\\
+ \eta_k^2LE\left(\frac{\eta_kLE}{2}\left(\frac{1}{n} + \frac{3 E}{4}\right) + \frac{1}{n}\right)\sigma^2
+ \eta_k\left(1 + \eta_k LE\right)\gamma.
\end{gathered}
\end{equation}
\end{proof}
We can now use this result to obtain convergence guarantees as stated in Theorem \ref{Theorem1}. The following Lemma \ref{LemmaNonBCD} helps us bound term $D$ in Lemma \ref{EpsilonBound}, which we can think of as a bound on the expected error between the gradient at $w_k$ with respect to the overall gradient computed at a given communication round $k$ and local step $\tau$.

\begin{lemma}\label{LemmaNonBCD} With $\eta_k LE \leq \frac{1}{3}$, we get that:
\begin{equation}\label{LemmaBound}
\begin{gathered}
    \sum_{\tau=0}^{E-1}\E\left[\|\nabla f \left(w_k\right) - u_{k,\tau}\|^2\right] \leq \eta_k^2 L^2E^2 \sum_{\tau=0}^{E-1}\E\left[\|u_{k,\tau}\|^2\right] + \\
    \frac{9\eta_k^2L^2E^3}{2n}\sum_{i\in\left[n\right]}\E\left[\|\nabla f_i \left(w_k\right)\|^2\right] + \eta_k^2 L^2 E^2\left(\frac{1}{n} + \frac{3E}{4}\right)\sigma^2.
\end{gathered}
\end{equation}
\end{lemma}

\begin{proof} 
Knowing that:

\begin{equation}\label{NonBCD1}
\begin{gathered}
\E\left[\|\nabla f \left(w_k\right) - u_{k,\tau}\|^2\right] = \\
\E\left[\|\nabla f\left(w_k\right) - \nabla f\left(\bar{w}_{k,\tau}\right) + \nabla f\left(\bar{w}_{k,\tau} - u_{k,\tau} \right)\|^2\right],
\end{gathered}
\end{equation}
\\
\begin{equation}\label{NonBCD2}
\begin{gathered}
\leq 2\E\left[\|\nabla f \left(w_k\right) - \nabla f \left(\bar{w}_{k,\tau}\right)\|^2\right] + \\
2\E\left[\|\nabla f \left(\bar{w}_{k,\tau}\right) - u_{k,\tau}\|^2\right],
\end{gathered}
\end{equation}
\\
\begin{equation}\label{NonBCD3}
\begin{gathered}
\leq 2L^2\E\left[\|w_k - \bar{w}_{k,\tau}\|^2\right] +\\
2\E\left[\|\frac{1}{n}\sum_{i\in\left[n\right]}\underbrace{\left(\nabla f_i \left(\bar{w}_{k,\tau}\right) - \nabla f_i\left(w_{k,\tau}^{\left(i\right)}\right)\right)}_\text{$ = -\tilde{e}_{k,\tau}^{\left(i\right)}$}\|^2\right],
\end{gathered}
\end{equation}

\begin{equation}\label{NonBCD4}
\leq 2 \eta_k^2L^2 \E\left[\|\sum_{t=0}^{\tau - 1} \hat{u}_{k,t}\|^2\right] + \frac{2}{n}\sum_{i\in\left[n\right]} \E\left[\|\tilde{e}_{k,\tau}^{\left(i\right)}\|^2\right],
\end{equation}

\begin{equation}\label{NonBCD5}
\begin{gathered}
\leq 2 \eta_k^2L^2 \left(\tau \sum_{t=0}^{\tau - 1} \E\left[\| u_{k,t}\|^2\right] + \frac{\tau \sigma^2}{n}\right) + \\
\frac{2L^2}{n}\sum_{i\in\left[n\right]} \E\left[\|w_{k,\tau}^{\left(i\right)} - \bar{w}_{k,\tau}^{\left(i\right)}\|^2\right].
\end{gathered}
\end{equation}

Where \eqref{NonBCD3} comes about by $f$'s $L$-smoothness and from the way $u_{k,\tau}$ is defined. Equation \eqref{NonBCD4} comes about from equation \eqref{GlobalUpdate},
and equation \eqref{NonBCD5} results from $f_i$'s $L$-smoothness as well as equation \eqref{OverallGradBound}. 
 
However:

\begin{equation}\label{UpdateDiff}
\begin{gathered}
\sum_{i\in\left[n\right]} \E\left[\|w_{k,\tau}^{\left(i\right)} - \bar{w}_{k,\tau}\|^2\right] =\\
\sum_{i\in\left[n\right]} \E\left[\|\left(w_{k,0}^{\left(i\right)} - \eta_k \sum_{t=0}^{\tau-1}\hat{u}_{k,t}^{\left(i\right)}\right) - \left(\bar{w}_{k,0} - \eta_k \sum_{t = 0}^{\tau -1 }\hat{u}_{k,t}\right)\|^2\right],
\end{gathered}
\end{equation}

\begin{equation}\label{UpdateDiff2}
= \eta_k^2 \sum_{i\in\left[n\right]}\E\left[\|\sum_{t=0}^{\tau - 1} \hat{u}_{k,t} - \sum_{t=0}^{\tau-1}\hat{u}_{k,t}^{\left(i\right)}\|^2\right],
\end{equation}

\begin{equation}\label{UpdateDiff3}
\leq \eta_k^2 \tau \sum_{i\in\left[n\right]}\sum_{t=0}^{\tau-1}\E\left[\|\hat{u}_{k,t} - \hat{u}_{k,t}^{\left(i\right)}\|^2\right],
\end{equation}
\begin{equation}\label{UpdateDiff4}
= \eta_k^2 \tau \sum_{i\in\left[n\right]}\sum_{t=0}^{\tau-1}\E\left[\|\hat{u}_{k,t}\|^2 + \|\hat{u}_{k,t}^{\left(i\right)}\|^2 - 2 \left<\hat{u}_{k,t},\hat{u}_{k,t}^{\left(i\right)}\right>\right].
\end{equation}

Now, equation \eqref{UpdateDiff2} results from $w_{k,0}^{\left(i\right)} = w_k \forall\, i \in \left[n\right]$, because of this $\bar{w}_{k,0} = w_k$. Additionally, $\hat{u}_{k,t} = \frac{1}{n}\sum_{i\in\left[n\right]}\hat{u}_{k,\tau}^{\left(i\right)}$, \eqref{UpdateDiff4} simplifies to:

\begin{equation}\label{UpdateDiff5}
\begin{gathered}
\sum_{i\in\left[n\right]}\E\left[\|w_{k,\tau}^{\left(i\right)} - \bar{w}_{k,\tau}\|^2\right] \\
\leq \eta_k^2 \tau \sum_{i\in\left[n\right]}\sum_{t=0}^{\tau-1}\left(\E\left[\|\hat{u}_{k,\tau}^{\left(i\right)}\|\right] - \E\left[\|\hat{u}_{k,t}\|^2\right]\right),
\end{gathered}
\end{equation}

\begin{equation}\label{UpdateDiff6}
\leq \eta_k^2 \tau \sum_{i\in\left[n\right]}\sum_{t=0}^{\tau-1} \E\left[\|\hat{u}_{k,t}^{\left(i\right)}\|^2\right],
\end{equation}

\begin{equation}\label{UpdateDiff7}
\leq \eta_k^2 \tau \sum_{i\in\left[n\right]}\sum_{t=0}^{\tau-1} \left(\E\left[\|\nabla f_i\left(w_{k,t}^{\left(i\right)}\right)\|^2\right] + \sigma^2 \right).
\end{equation}

Using the result from Lemma \ref{IndGrads}, we get:

\begin{equation}\label{UpdateDiff8}
\begin{gathered}
\sum_{i\in\left[n\right]}\E\left[\|w_{k,\tau}^{\left(i\right)} -w_{k,\tau}\|^2\right] \\
\leq \frac{9\eta_k^2\tau^2}{8}\sum_{i\in\left[n\right]}\left(2\E\left[\|\nabla f_i \left(w_k\right)\|^2\right] + \sigma^2\right).
\end{gathered}
\end{equation}

Plugging the result from \eqref{UpdateDiff8} into \eqref{NonBCD5}, we obtain:

\begin{equation}\label{NonBCD6}
\begin{gathered}
    \E\left[\|\nabla f\left(w_k\right) - u_{k,\tau}\|^2\right] \leq 2\eta_k^2L^2\left(\tau\sum_{t=0}^{\tau-1}\E\left[\|u_{k,t}\|^2\right] + \frac{\tau\sigma^2}{n}\right) 
    +\\
    \frac{9 \eta_k^2 L^2\tau^2}{4n}\sum_{i\in\left[n\right]}\left(2\E\left[\|\nabla f_i\left(w_k\right)\|^2\right] + \sigma^2\right),
\end{gathered}
\end{equation}

\begin{equation}\label{NonBCD7}
\begin{gathered}
= 2\eta_k^2L^2\tau\sum_{t=0}^{\tau-1}\E\left[\|u_{k,t}\|^2\right] +\\
\frac{9 \eta_k^2 L^2\tau^2}{2n}\sum_{i\in\left[n\right]}\E\left[\|\nabla f_i\left(w_k\right)\|^2\right] +
\eta_k^2L^2\tau\sigma^2\left(\frac{2}{n} + \frac{9\tau}{4}\right).
\end{gathered}
\end{equation}

Now, summing equation \eqref{NonBCD7} for $\tau \in \{0,...,E-1\}$, we get:

\begin{equation}\label{NonBCD8}
\begin{gathered}
\sum_{\tau=0}^{E-1}\E\left[\|\nabla f\left(w_k\right) - u_{k,\tau}\|^2\right] \leq \eta_k^2L^2E^2\sum_{\tau=0}^{E-1}\E\left[\|u_{k,\tau}\|^2\right] + \\
\frac{9\eta_k^2L^2E^3}{2n}\sum_{i\in\left[n\right]}\E\left[\|\nabla f_i\left(w_k\right)\|^2\right] + \\\eta_k^2L^2E^2\left(\frac{1}{n} + \frac{3 E}{4}\right)\sigma^2.
\end{gathered}
\end{equation}
\end{proof}
We can now use this bound on term $D$ from Lemma \ref{EpsilonBound} and continue with our analysis. However, in stating Lemma \ref{LemmaNonBCD}, we needed the results from Lemma \ref{IndGrads} which provide an upper bound on the expected value of the client's gradients at a given communication round $k$ and local step $\tau$.

\begin{lemma}\label{IndGrads}
With $\eta_k LE \leq \frac{1}{3}$, we get that:
\begin{equation}\label{IndGradBound}
\sum_{t=0}^{\tau -1 }\E\left[\|\nabla f_i \left(w_{k,t}^{\left(i\right)} \right)\|^2\right] \leq \frac{\tau}{8}\left(18\E\left[\|\nabla f_i \left(w_k\right)\|^2\right] + \sigma^2\right).
\end{equation}
\end{lemma}

\begin{proof}
We have that:

\begin{equation}\label{IndGradProof}
\begin{gathered}
    \E\left[\|\nabla f_i\left(w_{k,t}^{\left(i\right)}\right)\|^2\right] = \\
    \E\left[\|\nabla f_i\left(w_{k,t}^{\left(i\right)}\right) - \nabla f_i\left(w_k\right) + \nabla f_i\left(w_k\right)\|^2\right],
\end{gathered}
\end{equation}

\begin{equation}\label{IndGradProof1}
    \leq 2\E\left[\|\nabla f_i \left(w_k\right)\|^2\right] + 2\E\left[\|\nabla f_i\left(w_{k,t}^{\left(i\right)}\right) - \nabla f_i\left(w_k\right)\|^2\right],
\end{equation}

\begin{equation}\label{IndGradProof2}
    \leq 2\E\left[\|\nabla f_i \left(w_k\right)\|^2\right] + 2L^2\E\left[\|w_{k,t}^{\left(i\right)} - w_k\|^2\right].
\end{equation}

But:

\begin{equation}\label{IndGradProof3}
\begin{gathered}
\E\left[\|w_k - w_{k,t}^{\left(i\right)}\|^2\right] = \E\left[\|\eta_k \sum_{t'=0}^{t-1} \nabla \tilde{f}_i\left(w_{k,t'}^{\left(i\right)};\mathcal{B}_{k,t'}^{\left(i\right)}\right)\|^2\right] \\
\leq \eta_k^2t\sum_{t'=0}^{t-1} \E\left[\|\eta_k \sum_{t'=0}^{t-1} \nabla \tilde{f}_i\left(w_{k,t'}^{\left(i\right)};\mathcal{B}_{k,t'}^{\left(i\right)}\right)\|^2\right]\\
\leq \eta_k^2t \sum_{t'=0}^{t-1} \left(\E\left[\|\nabla f_i\left(w_{k,t'}^{\left(i\right)}\right)\|^2\right] + \sigma^2\right).
\end{gathered}
\end{equation}

Plugging this on \eqref{IndGradProof2}, we obtain:

\begin{equation}\label{IndGradProof3-2}
\begin{gathered}
\E\left[\|\nabla f_i \left( w_{k,t}^{\left(i\right)}\right)\|^2\right] \leq 2\E\left[\|\nabla f_i\left(w_k\right)\|^2\right] + \\
2\eta_k^2L^2t\sum_{t'=0}^{t-1}\left(\E\left[\|\nabla f_i\left(w_{k,t'}^{\left(i\right)}\right)\|^2\right] + \sigma^2\right).
\end{gathered}
\end{equation}

Now, summing equation \eqref{IndGradProof3-2} by $t \in \{0, ..., \tau - 1\}$, we obtain:

\begin{equation}\label{IndGradProof4}
\begin{gathered}
\sum_{t=0}^{\tau - 1}\E\left[\|\nabla f_i\left(w_{k,t}^{\left(i\right)}\right)\|^2\right] \leq 2\tau\left(\E\left[\|\nabla f_i\left(w_k\right)\|^2\right]\right) + \\
2\eta_k^2L^2\sum_{t=0}^{\tau-1}\tau\sum_{t'=0}^{t-1}\left(\E\left[\|\nabla f_i\left(w_{k,t'}^{\left(i\right)}\right)\|^2\right] + \sigma^2\right),
\end{gathered}
\end{equation}

\begin{equation}\label{IndGradProof5}
\begin{gathered}
\leq 2\tau\left(\E\left[\|\nabla f_i\left(w_k\right)\|^2\right]\right) +\\ \eta_k^2L^2\tau^2\sum_{t=0}^{\tau-1}\left(\E\left[\|\nabla f_i\left(w_{k,t}^{\left(i\right)}\right)\|^2\right] + \sigma^2\right).
\end{gathered}
\end{equation}

Set $\eta_kLE \leq \frac{1}{3}$. Then:

\begin{equation}\label{IndGradProof6}
\begin{gathered}
\sum_{t=0}^{\tau - 1}\E\left[\|\nabla f_i\left(w_{k,t}^{\left(i\right)}\right)\|^2\right] \leq 2\tau\left(\E\left[\|\nabla f_i\left(w_k\right)\|^2\right]\right) + \\
\frac{1}{9}\sum_{t=0}^{\tau-1}\left(\E\left[\|\nabla f_i\left(w_{k,t}^{\left(i\right)}\right)\|^2\right] + \sigma^2\right).
\end{gathered}
\end{equation}

Simplifying and re-arranging equation \eqref{IndGradProof6}, we obtain: 

\begin{equation}\label{IndGradProof7}
\sum_{t=0}^{\tau-1}\E\left[\|\nabla f_i \left(w_{k,t}^{\left(i\right)}\right)\|^2\right] \leq \frac{\tau}{8}\left(18\E\left[\|\nabla f_i\left(w_k\right)\|^2\right] + \sigma^2\right).
\end{equation}

We can now use the results from Lemma \ref{IndGrads} to help provide a bound for Lemma \ref{LemmaNonBCD}.
\end{proof}

\end{document}